\documentclass[conference]{IEEEtran}
\IEEEoverridecommandlockouts
\usepackage{cite}
\usepackage{amsmath,amssymb,amsfonts}
\usepackage{graphicx}
\usepackage{textcomp}
\usepackage{xcolor}

\usepackage{booktabs}       
\usepackage{amsfonts}       
\usepackage{nicefrac}       
\usepackage{microtype}      
\usepackage{lipsum}
\usepackage[ruled,vlined]{algorithm2e}
\DeclareMathOperator*{\argmax}{argmax}

\usepackage{amsthm}

\usepackage{graphicx}
\usepackage{subcaption}
\usepackage{xcolor}
\usepackage{cite}

\newcommand{\pgml}{\mathcal{P}_{G_{ML}}}
\newcommand{\dml}{D_{ML}}
\newcommand{\gml}{G_{ML}}
\newcommand{\method}{$\mathsf{GANDEC}$}

\allowdisplaybreaks

\newtheorem{theorem}{Theorem}

\newtheorem{lemma}{Lemma}
\newtheorem{proposition}{Proposition}

\DeclareMathOperator{\sign}{sign}

\def\BibTeX{{\rm B\kern-.05em{\sc i\kern-.025em b}\kern-.08em
    T\kern-.1667em\lower.7ex\hbox{E}\kern-.125emX}}
\begin{document}

\title{Adversarial Neural Networks for \\ Error Correcting Codes}

\author{%
	\IEEEauthorblockN{Hung T. Nguyen\IEEEauthorrefmark{1},
		Steven Bottone\IEEEauthorrefmark{2},
		Kwang Taik Kim\IEEEauthorrefmark{3},
		Mung Chiang\IEEEauthorrefmark{3},
		and H. Vincent Poor\IEEEauthorrefmark{1}}
	\IEEEauthorblockA{\IEEEauthorrefmark{1}%
		Princeton University,
		\{hn4,poor\}@princeton.edu}
	\IEEEauthorblockA{\IEEEauthorrefmark{2}%
		Northrop Grumman Corporation,
		steven.bottone@ngc.com}
	\IEEEauthorblockA{\IEEEauthorrefmark{3}%
		Purdue University,
		\{kimkt,chiang\}@purdue.edu}
}

\maketitle

\begin{abstract}
Error correcting codes are a fundamental component in modern day communication systems, demanding extremely high throughput, ultra-reliability and low latency. Recent approaches using machine learning (ML) models as the decoders offer both improved performance and great adaptability to unknown environments, where traditional decoders struggle. We introduce a general framework to further boost the performance and applicability of ML models. We propose to combine ML decoders with a competing discriminator network that tries to distinguish between codewords and noisy words, and, hence, guides the decoding models to recover transmitted codewords. Our framework is game-theoretic, motivated by generative adversarial networks (GANs), with the decoder and discriminator competing in a zero-sum game. The decoder learns to simultaneously decode and generate codewords while the discriminator learns to tell the differences between decoded outputs and codewords. Thus, the decoder is able to decode noisy received signals into codewords, increasing the probability of successful decoding. We show a strong connection of our framework with the optimal maximum likelihood decoder by proving that this decoder defines a Nash's equilibrium point of our game. Hence, training to equilibrium has a good possibility of achieving the optimal maximum likelihood performance. Moreover, our framework does not require training labels, which are typically unavailable during communications, and, thus, seemingly can be trained online and adapt to channel dynamics. To demonstrate the performance of our framework, we combine it with the very recent neural decoders and show improved performance compared to the original models and traditional decoding algorithms on various codes.
\end{abstract}


\section{Introduction}
New data-intensive use cases, e.g., autonomous driving, high-precision industrial plants, demand high throughput, ultra-reliable, and low latency communications for the fifth \cite{TS38.104} and subsequent generation networks. This demand requires communication algorithms to be performing and highly adaptive to any changes in the environment. For example, iterative error correcting codes' decoders need to operate with minimal number of iterations, while providing peak decoding performance at any time subject to channel variations. However, little is theoretically known about optimal parameter tuning under strict iterative constraints and it is usually established using heuristics via simulations or pessimistic bounds in practice \cite{Balatsoukas--Studer2019}, which limit the efficacy and applicability of those methods.



Machine learning approaches have been introduced recently as a promising candidate to address these challenges. As an exemplar, deep unfolding networks \cite{hershey2014deep} show great performance and channel adaptability across various tasks, such as decoders for error correcting codes \cite{nachmani2018deep,lugosch2017neural,lugosch2018learning,zhang2020iterative}, symbol detection \cite{shlezinger2020deepsic,shlezinger2019viterbinet,samuel2019learning,he2018model}, and MIMO precoding \cite{balatsoukas2019neural}, by combining the learning power of emerging deep neural networks (DNNs) with rigorous analytical understanding of model-based iterative algorithms in communications.
On one hand, deep unfolding removes the assumptions of existing algorithms on channel conditions and can deal with both linear and nonlinear channels, while retaining the interpretability of well understood iterative algorithms thanks to having the same structure \cite{yedidia2001characterization,kuck2020belief}.
On the other hand, deep unfolding can speed up computation by reducing the number of iterations required to achieve good performance due to the impressive learnability of DNNs, requiring only a few layers corresponding to a small number of iterations in existing iterative algorithms.
However, there is not a formal theoretical result about deep unfolding networks. All the properties of underlying algorithm are informally carried over to deep unfolding model due to structure sharing. 

We propose a framework to further boost the performance of ML models for the task of decoding error correcting codes by leveraging the same idea in generative adversarial networks (GANs) of having two models competing in a game and together improving their performances (Subsection~\ref{subsec:model}).
In our framework, we have an ML decoder learning to decode received (noisy) signals and design a discriminator to distinguish between (valid) codewords and decoder's outputs.
The goal of the decoder is 
to deceive the discriminator to believe in that its decoded sequences are an actual codewords.
As the discriminator gets better in recognizing codewords, the decoder must learn to generate codeword-like sequences or to avoid non-codeword regions (typically significantly large). This can be seen as introducing an insightful constraint on the decoder's outputs in addition to existing ones on the decoder, e.g., structural constraints on deep unfolding networks. Thus, intuitively the decoder finds a codeword that is closest to the noisy signal with respect to some hidden metric defined by the discriminator, i.e., matching the distribution of true codewords.

Our game-theoretic framework, combining an ML decoder and a discriminator, enjoys multiple advantageous properties.
We first prove that the optimal maximum likelihood decoder of a code along with a custom trained discriminator form a Nash's equilibrium of our game model (Subsection~\ref{subsec:approach_ml}).
This establishes a strong connection between our model and the optimal maximum likelihood decoder: if we train the model till convergence (equilibrium), it has a good chance of approaching the optimal performance. Furthermore, applying our framework on the deep unfolding decoders \cite{lugosch2018learning,lugosch2017neural,zhang2020iterative}, that unfold several iterations of the belief propagation (BP) algorithm into a deep network, we show that a simple variant of our model also preserves all the fixed points of the original BP decoder (Subsection~\ref{subsec:bp_fixed}).
These fixed points correspond to solutions of minimizing the Bethe free energy \cite{yedidia2003understanding}, which is vital for inference tasks on factor graphs. More importantly, unlike existing models, our framework does not rely on training labels, i.e., transmitted signals, which are usually unavailable during communications, and thus has the capability of online training and adapts to any channel variations. We experimentally verify that our framework leads to improved decoding performance compared with the original deep unfolding models and the BP algorithms on various error correcting codes (Section~\ref{sec:exps}).

Throughout the paper, we use the unfolded BP networks \cite{lugosch2018learning,lugosch2017neural} (details in Section~\ref{sec:prelim}) as the demonstrating ML models in our framework for both the theory and experiments.

\section{Preliminaries}
\label{sec:prelim}
As the primary example of machine learning model in our paper, we describe in details the recent approach of unfolded belief propagation (BP) decoder. 
We first review the well known belief propagation (BP) decoder for error correcting codes and the deep neural network model that is constructed by unfolding this BP method.
Specifically, we summarize sum-product and min-sum algorithms along with their deep unfolding neural network models.

We consider the discrete-time additive white Gaussian noise (AWGN) channel model. In transmission time $i \in [1:n]$, the channel output is $Y_i = gX_i + Z_i$. where $g$ is the channel gain, $X_i \in \{ \pm 1\}$, $\{ Z_i \}$ is a WGN ($\sigma^2$) process, independent of the channel input $X^n = x^n(m)$, $m \in [1:2^{k}]$, and $k$ is the message length.


\begin{figure}[!h]
	\centering
	\includegraphics[width=\linewidth]{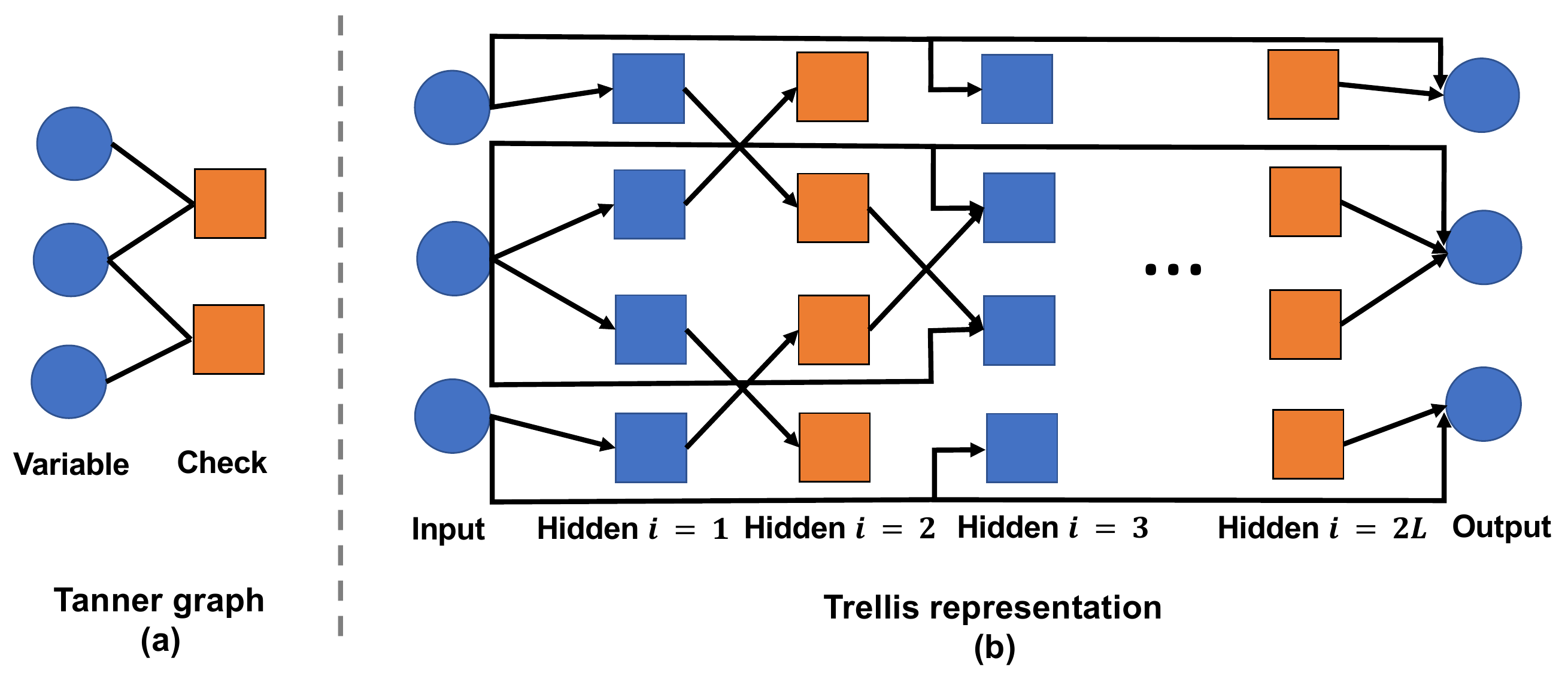}
	\caption{Tanner graph of a code and unfolded BP architecture based on code Trellis.}
	\label{fig:bp}
\end{figure}
\subsection{Belief Propagation (BP) Decoder}

Given a linear error correcting code $C$ of block-length $n$, characterized by a parity check matrix $H$ of size $n \times (n-k)$, every codeword $x \in C$ satisfies $H^T x = 0$, where $H^T$ denotes the transpose of $H$ and $0$ the vector of $n-k$ zeros.
The BP decoder for a linear code can be constructed from the Tanner graph of the parity check matrix.
The Tanner graph contains $n-k$ check nodes and $n$ variable nodes arranged as a bipartite graph with edges between check and variable nodes, representing messages passing between these nodes.
If $H_{vc} = 1$, there is an edge between variable node $v$ and check node $c$.
A simple example is given in Figure~\ref{fig:bp}(a) with $n = 3$ variable nodes and $(n-k) = 2$ check nodes.
The BP decoder consists of multiple iterations.
Each iteration corresponds to one round of messaging passing between variable and check nodes.
In iteration $i$, on edge $e = (v,c)$, the message $x_{i, (v,c)}$ passed from variable node $v$ to the check node $c$ is given by
\begin{align}
	\label{eq:m_vc_bp}
	x_{i, e = (v,c)} = l_v + \sum_{e' = (c',v), c' \neq c} x_{i-1, e'},
\end{align}
where $l_v$ is the log-likelihood ratio of the channel output corresponding to the $v$th codeword bit $C_v$, i.e., $l_v = \log\frac{\Pr(C_v = 1 | y_v)}{\Pr(C_v = 0 | y_v)}$.

In the sum-product algorithm, the message $x_{i, (c,v)}$ from check node $c$ to variable node $v$ in iteration $i$ is computed as
\begin{align}
	\label{eq:m_cv_bp}
	x_{i, e = (c,v)} = 2 \tanh^{-1} \Bigg( \prod_{e' = (v',c), v' \neq v} \tanh \frac{x_{i,e'}}{2} \Bigg).
\end{align}
For numerical stability, the messages going into a node in both Eq.~(\ref{eq:m_vc_bp}) and Eq.~(\ref{eq:m_cv_bp}) are normalized at every iteration, and all the messages are also truncated within a fixed range.

Note that the computation in Eq.~(\ref{eq:m_cv_bp}) involves repeated multiplications and hyperbolic functions, which may lead to numerical problems in actual implementations. Even with truncation of message values within a certain range, e.g., commonly $[-10, 10]$, we still observe numerical problems when running the sum-product algorithm.



The min-sum algorithm uses a ``min-sum'' approximation of the above message as follows:
\begin{align}
	x_{i, e = (c,v)} = \min_{e' = (v',c), v' \neq v} |x_{i,e'}| \prod_{e' = (v',c), v' \neq v} \sign(x_{i,e'}). \nonumber
\end{align}
Suppose the BP decoder runs $L$ iterations. Then the $v$th output after the final iteration is given by
\begin{align}
	\label{eq:output}
	o_v = l_v + \sum_{e' = (c',v)} x_{L,e'}.
\end{align}

\subsection{Unfolded BP Models}
The BP decoder has an equivalent Trellis representation that motivates the unfolding neural network architecture.
A simple example is provided in Figure~\ref{fig:bp}(b).
Each BP iteration $i$ unfolds into 2 hidden layers, $2i-1$ and $2i$, corresponding to two passes of messages from variable to check nodes and from check to variable nodes.
The number of nodes in each hidden layer is the same as the number of edges in the Tanner graph.
Each node computes a message sent through each edge.
Thus, there are $2\times L$ hidden layers in this Trellis representation in addition to input and output layers.
The input takes log-likelihood ratios of the received signal, and the output computes Eq.~\eqref{eq:output}.

The unfolding neural network model shares the same architecture as the Trellis representation and adds trainable weights to the edges.
In other words, in odd layer $i$, the computation at node $i, e = (v,c)$, performs the following operation (activation function):
\begin{align}
	\label{eq:m_vc}
	x_{i, e = (v,c)} = w_{i,v}l_v + \sum_{e' = (c',v), c' \neq c} w_{i,e,e'}x_{i-1, e'}.
\end{align}
Note that trainable weights $w_{i,v}$ and $w_{i,e,e'}$ are introduced in the computation.
These weights can be trained and address various challenging obstacles with the BP method: 1) reduce the large number of iterations needed in BP, leading to higher efficiency; 2) improve decoding performance and possibly approach optimal criterion, thanks to the ability to minimize the effects of short cycles in the Tanner graph known to cause difficulties for BP decoders.

In even layer $i$, the computation at node $i, e = (c,v)$ is similar to that in the regular BP algorithm:
\begin{align}
	\label{eq:m_cv}
	x_{i, e = (c,v)} = 2 \tanh^{-1} \Bigg( \prod_{e' = (v',c), v' \neq v} \tanh \frac{x_{i-1,e'}}{2} \Bigg).
\end{align}
For numerical stability, we usually apply truncation of messages after every layer.

Nodes at the output layer perform
\begin{align}
	\label{eq:o_v}
	o_v = \sigma \Bigg(w_{2L + 1,v}l_v + \sum_{e' = (c',v)} w_{2L+1,v,e'} x_{2L,e'} \Bigg),
\end{align}
where $\sigma(\cdot)$ is the sigmoid function to obtain probability from the log-likelihood ratio representation.
Note that trainable weights $w_{2L + 1,v}$ and $w_{2L+1,v,e'}$ are also introduced here.

\section{Proposed Framework}
\label{sec:alg}

Motivated by the recent advances in generative adversarial networks (GANs), we propose to combine a decoder model, e.g., unfolded BP network model, with a discriminator that distinguishes between codewords following the uniform distribution and decoded outputs of the decoder.
The decoder and discriminator compete in a zero-sum game: the decoder tries to fool the discriminator into thinking that its outputs are codewords, i.e., maximizing the classification error; on the other hand, the discriminator is a classification model that tries to tell which input sequences are original codewords and which are generated by the decoder.
Effectively, the two models improve simultaneously as the better decoder or discriminator gives better feedback to the other about how to improve to compete better.

\subsection{Architecture} \label{subsec:model}
\begin{figure}[!t]
	\centering
	\includegraphics[width=\linewidth]{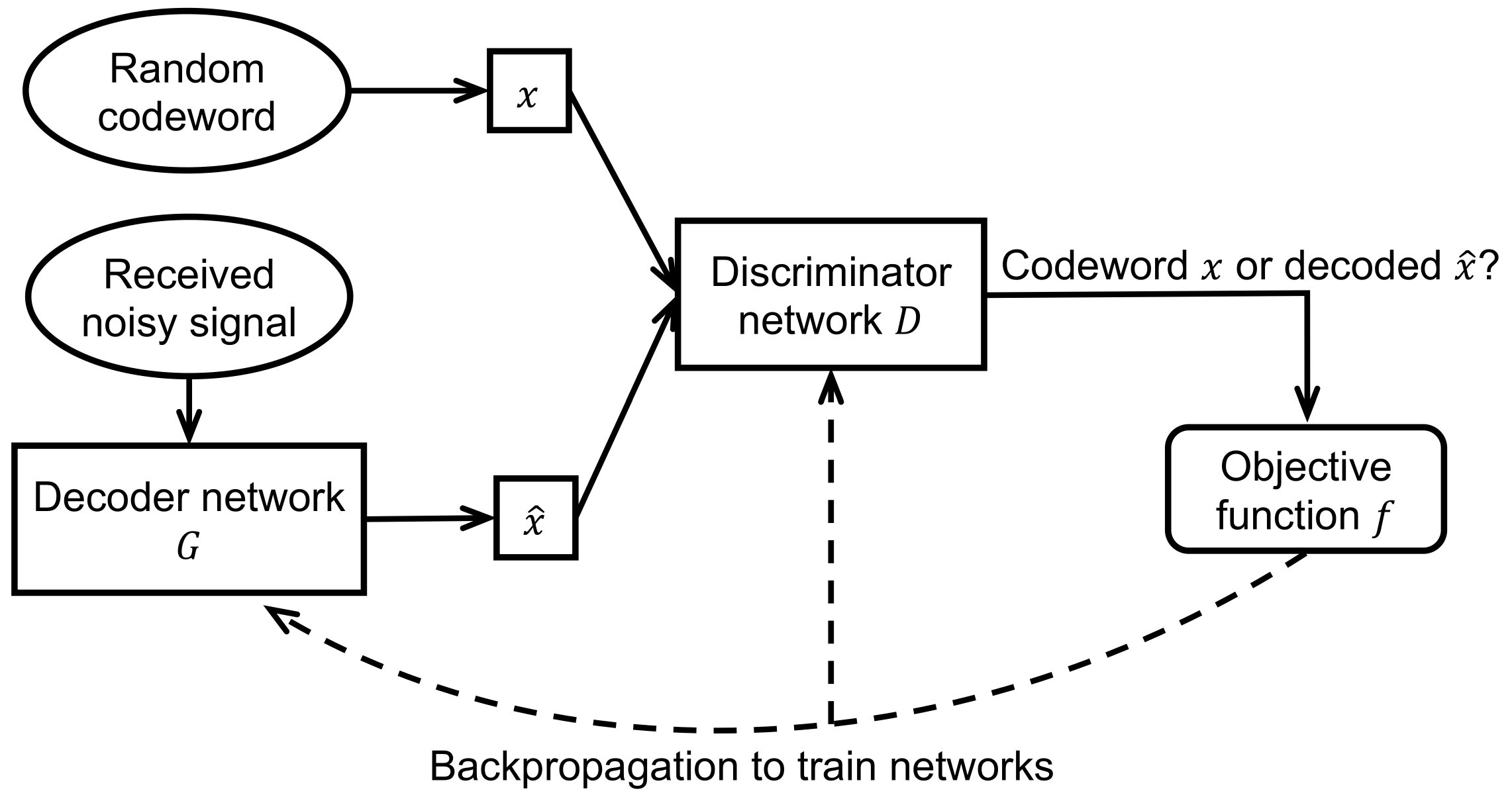}
	\caption{Proposed framework with decoder and discriminator networks competing and improving each other.}
	\label{fig:proposed_model}
\end{figure}

The proposed framework is described in Figure~\ref{fig:proposed_model}.
The decoder network $G$ takes the noisy signal from the channel and produces the decoded output $\hat x$, e.g., $o_v$ in Eq.~\ref{eq:o_v}. The discriminator network $D$ is a classification model and aims at discriminating between codewords $x$ and decoded outputs $\hat x$ from the decoder.
The function $f(G,D)$ evaluates how well the discriminator network performs.
Thus, the decoder's goal is the opposite of the discriminator's, minimizing $f(G,D)$.

Let $\mathcal{U}$ denote the uniform distribution generating random codewords and $\mathcal{P}_G$ the resulting distribution of the decoder's outputs.
Here, we slightly abuse notation and use $\mathcal{U}$ and $\mathcal{P}_G$ to denote both a distribution and a probability mass function.
We use the following common objective function in GANs:
\begin{align}
	\label{eq:f_gd}
  f(G,D) = \mathbb{E}_{x \sim \mathcal{U}} [\log D(x)] + \mathbb{E}_{y \sim \mathcal{P}_G} [\log(1 - D(y))].
\end{align}
This function measures how well the discriminator classifies codewords from decoded outputs.
In other words, the discriminator tries to maximize this function while the decoder aims at minimizing it.
Hence, the optimization problem is to find $G^*$ and $D^*$ that solve the following min-max problem:
\begin{align}
	\label{eq:opt_problem}
  \min_{G}\max_{D} \mathbb{E}_{x \sim \mathcal{U}} [\log D(x)] + \mathbb{E}_{y \sim \mathcal{P}_G} [\log (1 - D(y))],
\end{align}
where $G$ is the decoder network with parameters $w_{i,v}, w_{i,e,e'}$ and $w_{2L+1,v,e'}$, and $D$ is a classification model within some class of models.
As we will show later, with different classes of classification models, we can prove various results regarding the equilibrium between decoder and discriminator with respect to the min-max game in Eq.~(\ref{eq:opt_problem}).

\subsection{Theoretical Properties}

\subsubsection{Approaching the Optimal Maximum Likelihood Performance}
\label{subsec:approach_ml}
We first prove that our model has a strong connection with the optimal maximum likelihood decoder, which is computationally intractable in practice due to its hardness.
We show that the maximum likelihood decoder along with a discriminator network trained with this decoder is, in fact, an Nash's equilibrium of our min-max game defined in Eq.~(\ref{eq:opt_problem}).
Let $G_{ML}$ represent the optimal maximum likelihood decoder and $D_{ML}$ be defined as follows:
\begin{align}
	\label{eq:dml_prob}
  D_{ML} = \argmax_{D} \mathbb{E}_{x \sim \mathcal{U}} [\log D(x)] + \mathbb{E}_{y \sim \pgml} [\log (1 - D(y))].
\end{align}
Note that $D$ is in the class of all possible classification models.

From \cite{goodfellow2014generative}, we have the following lemma characterizing $\dml$ based on $\gml$.
\begin{lemma}
	\label{lem:dml_opt}
	Let $D$ be from the class of all classification models.
	The solution $\dml$ for problem (\ref{eq:dml_prob}) with respect to fixing $G$ to be $\gml$, i.e., having distribution $\pgml$, is
	\begin{align}
    \dml(x) = \frac{\mathcal{U}(x)}{\mathcal{U}(x) + \pgml(x)}, ~ x \in C. \nonumber
	\end{align}
\end{lemma}

Based on the above lemma, we establish the following theorem showing that $\gml$ and $\dml$ form an equilibrium of our min-max problem in (\ref{eq:opt_problem}).
\begin{theorem}
	\label{theo:ml_equi}
	The optimal maximum likelihood decoder $\gml$ and the corresponding discriminator $\dml$ form an equilibrium for our min-max problem in (\ref{eq:opt_problem}).
	Specifically, we have
	\begin{align}
		\max_{D} f(\gml, D) \leq f(\gml, \dml) \leq \min_{G} f(G, \dml), \nonumber
	\end{align}
	where $f(G,D)$ is defined in Eq.~(\ref{eq:f_gd}) and $G$ is from the class of all decoder networks with different weight parameters, i.e.,  $w_{i,v}, w_{i,e,e'}$ and $w_{2L+1,v,e'}$ as defined in Eq.~(\ref{eq:m_vc}) and Eq.~(\ref{eq:o_v}).
\end{theorem}
\begin{proof}
	The first inequality $\max_{D} f(\gml, D) \leq f(\gml, \dml)$ follows immediately from the definition of $\dml$ in Lemma~\ref{lem:dml_opt} that $\dml$ is the optimal classification model with respect to the optimal maximum likelihood decoder $\gml$.

	To prove the second inequality $f(\gml, \dml) \leq \min_{G} f(G, \dml)$, consider the function $f(G, \dml)$.
	We have
  \begin{align}
		\label{eq:14}
		f(G, &\dml) \nonumber \\
		&= \mathbb{E}_{x \sim \mathcal{U}} [\log \dml(x)] + \mathbb{E}_{y \sim \mathcal{P}_G} [\log (1 - \dml(y))] \nonumber \\
		& = \mathbb{E}_{x \sim \mathcal{U}} \left[\log \frac{\mathcal{U}(x)}{\mathcal{U}(x) + \pgml(x)}\right] \nonumber \\
		& \qquad + \mathbb{E}_{y \sim \mathcal{P}_G} \left[\log \frac{\pgml(y)}{\mathcal{U}(y) + \pgml(y)}\right].
	\end{align}
	Regarding the optimization problem of $f(G, \dml)$ over $G$, the first term in Eq.~(\ref{eq:14}) is fixed for all $G$ and we can focus on the second term with $y \sim \mathcal{P}_G$.
	For any $y \not\in C$ and $\pgml(y) \neq 0$, $\log \frac{\pgml(y)}{\mathcal{U}(y) + \pgml(y)}$ achieves its maximum value of $0$ since $\log \frac{\pgml(y)}{\mathcal{U}(y) + \pgml(y)} \leq 0$.
	Therefore, to minimize $f(G,\dml)$, we can narrow our consideration to decoders that produce only codewords.
	Furthermore, due to the symmetry of linear block codes, $\pgml$ is the uniform distribution over codewords.
	Thus, for any decoder $G$ that only produces codewords, we have
	\begin{align}
		\label{eq:15}
    & \mathbb{E}_{y \sim \mathcal{P}_G} \left[\log \frac{\pgml(y)}{\mathcal{U}(y) + \pgml(y)}\right] \nonumber \\
		& = \mathbb{E}_{y \sim \pgml} \left[\log \frac{\pgml(y)}{\mathcal{U}(y) + \pgml(y)}\right] = \log \frac{1}{2}.
	\end{align}

	Taking the minimization over all decoder networks $G$ in Eq.~(\ref{eq:14}) and incorporating Eq.~(\ref{eq:15}) into Eq.~(\ref{eq:14}), we obtain
	\begin{align}
    \min_{G} f(G, \dml) & = \min_{G} \Big(\mathbb{E}_{x \sim \mathcal{U}} \left[\log \frac{\mathcal{U}(x)}{\mathcal{U}(x) + \pgml(x)}\right] \nonumber \\
		& \text{ \ \ } + \mathbb{E}_{y \sim \mathcal{P}_G} \left[\log \frac{\pgml(y)}{\mathcal{U}(y) + \pgml(y)}\right] \Big) \nonumber \\
		& = \mathbb{E}_{x \sim \mathcal{U}} \left[\log \frac{\mathcal{U}(x)}{\mathcal{U}(x) + \pgml(x)}\right] \nonumber \\
		& \text{ \ \ } + \min_{G} \left(\mathbb{E}_{y \sim \mathcal{P}_G} \left[\log \frac{\pgml(y)}{\mathcal{U}(y) + \pgml(y)}\right]\right) \nonumber \\
		& = \mathbb{E}_{x \sim \mathcal{U}} \left[\log \frac{\mathcal{U}(x)}{\mathcal{U}(x) + \pgml(x)}\right] \nonumber \\
		& \text{ \ \ } + \mathbb{E}_{y \sim \pgml} \left[\log \frac{\pgml(y)}{\mathcal{U}(y) + \pgml(y)}\right] \nonumber \\
		& = f(\gml, \dml),
	\end{align}
	which completes the proof of the second inequality and the theorem.
\end{proof}

The equilibrium point corresponding to the optimal maximum likelihood decoder suggests that if we train our framework till convergence to equilibrium, we can possibly achieve optimal maximum likelihood decoding performance.
This possibility is high for the unfolded BP network since its architecture is based on the Tanner graph that generates the code and, thus, the resulting network provides good approximations of the true likelihood over codewords.

\subsubsection{Preserving BP Fixed Points}
Specifically consider our example of unfolded BP network, we can say more about the properties of our framework that we can easily derive a variant to preserve all the \emph{fixed points} of the original BP algorithm.
We first define a fixed point of BP as a set of messages that remain the same values after updates in Eq.~(\ref{eq:m_vc_bp}) and Eq.~(\ref{eq:m_cv_bp}).
The BP method on factor graphs has been studied extensively in the literature \cite{yedidia2001characterization,yedidia2003understanding,kuck2020belief,satorras2020neural} and shown to have interesting properties.
Among those properties, an important one is the correspondence of BP fixed points to solutions of the Bethe free energy defined on a factor graph \cite{yedidia2001characterization,yedidia2003understanding}.
The solutions minimizing the Bethe free energy provide a lower-bound for the factor graph's partition function, which is useful for inferences on that graph.

We can derive a variant of our model that preserves all the fixed points of the BP method in a straightforward manner.
Specifically, let $\bar{x}_{i,e=(v,c)}$ be the updated message from variable node $v$ to check node $c$ for the original BP method in Eq.~(\ref{eq:m_vc_bp}).
The activation function in odd layer $i$ is modified to
\vspace{-0.05in}
\begin{align}
	\label{eq:m_vc_variant}
	x_{i, e = (v,c)} = & (\bar{x}_{i,e} - \bar{x}_{i-1,e}) \Big[w_{i,v}l_v \nonumber \\
	& + \sum_{e' = (c',v), c' \neq c} w_{i,e,e'}x_{i-1, e'}\Big] + \bar{x}_{i-1,e}.
\end{align}
The activation function on the even layers are the same as in the original model, i.e., Eq.~(\ref{eq:m_cv}).


We obtain the following result regarding the relation between BP fixed points and the variant of our model.
\begin{proposition}
	The simple variant of the deep unfolding network with activation function in Eq.~(\ref{eq:m_vc_variant}) preserves all the fixed points of the original belief propagation (BP) method.
\end{proposition}
\begin{proof}
	The proof is evident from the definition of fixed points.
	The set of $\bar{x}_{i,e=(v,c)}$ and $\bar{x}_{i,(e = (c,v))}$ is a fixed point of BP if $\bar{x}_{i,e=(v,c)} = \bar{x}_{i-1,e=(v,c)}$ and $\bar{x}_{i,(e = (c,v))} = \bar{x}_{i-1,(e = (c,v))}$ for all $(c,v)$ and $(v,c)$.
	Plugging these conditions into the activation functions of our model variant in Eq.~(\ref{eq:m_cv}) and Eq.~(\ref{eq:m_vc_variant}), we also obtain the same conditions for fixed points that $x_{i,e=(v,c)} = x_{i-1,e=(v,c)} = \bar{x}_{i-1,e=(v,c)}$ and $x_{i,(e = (c,v))} = x_{i-1,(e = (c,v))} = \bar{x}_{i-1,(e = (c,v))}$.
	Hence, any fixed point of BP is also a fixed point of our model.
\end{proof}

\begin{figure*}[h]
	\centering
	\begin{subfigure}[b]{0.42\linewidth}
		\includegraphics[width=\linewidth]{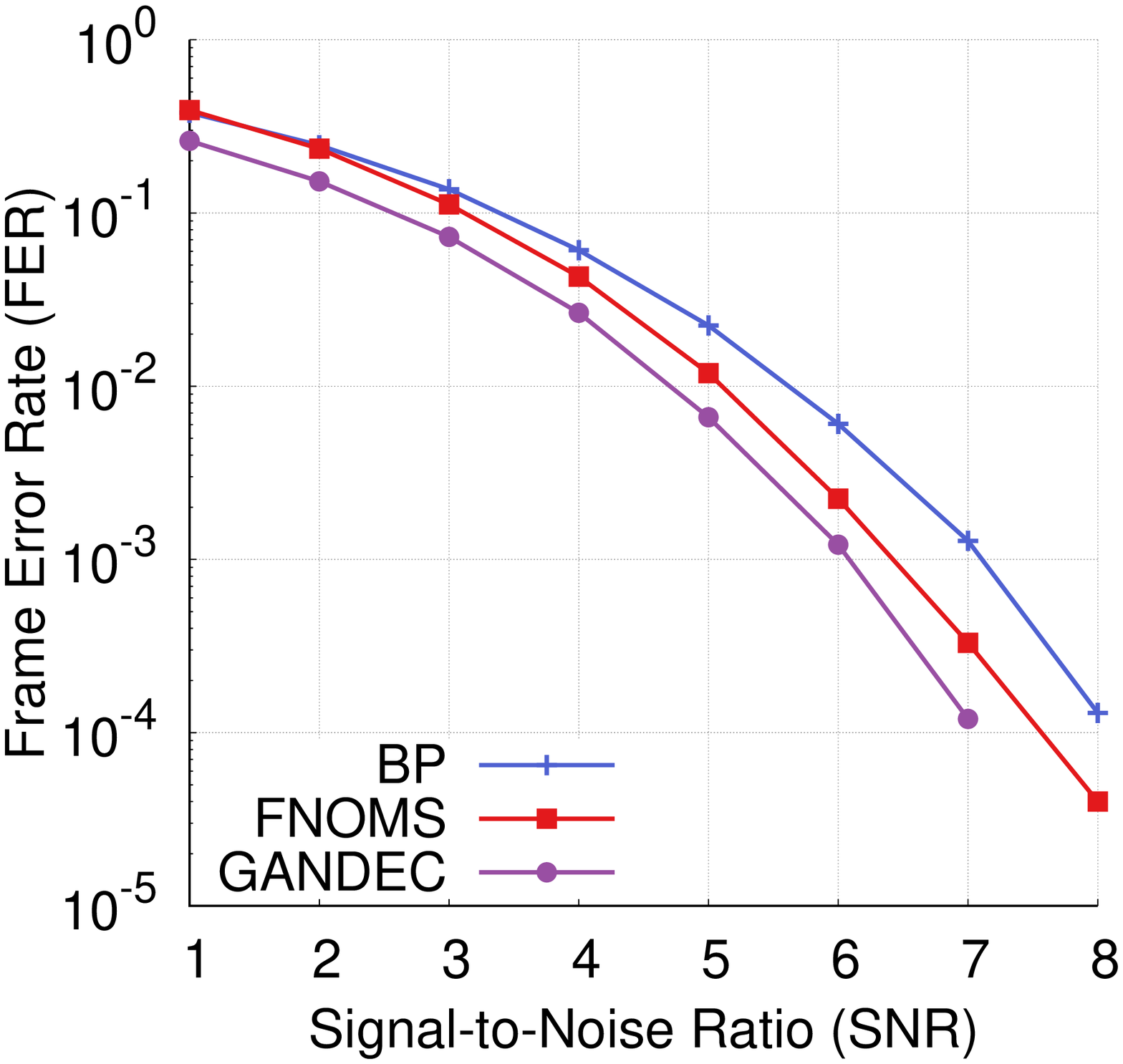}
		\caption{BCH(15,11)}
	\end{subfigure}
	\begin{subfigure}[b]{0.42\linewidth}
		\includegraphics[width=\linewidth]{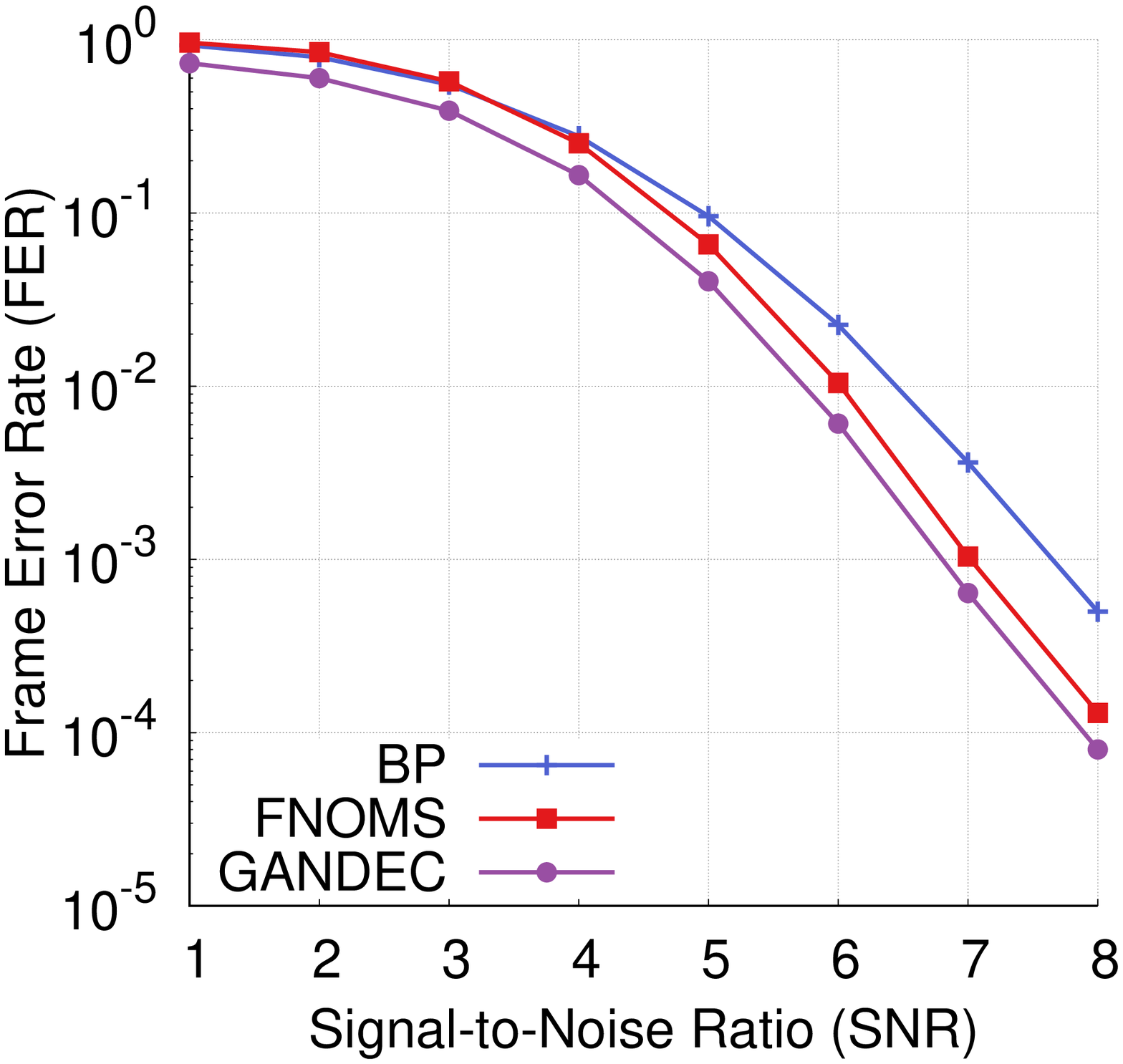}
		\caption{BCH(63,45)}
	\end{subfigure}
	\caption{Performance of different decoders on BCH codes.}
	\label{fig:bch_fer}
\end{figure*}
\begin{figure*}[h]
	\centering
	\begin{subfigure}[b]{0.42\linewidth}
		\includegraphics[width=\linewidth]{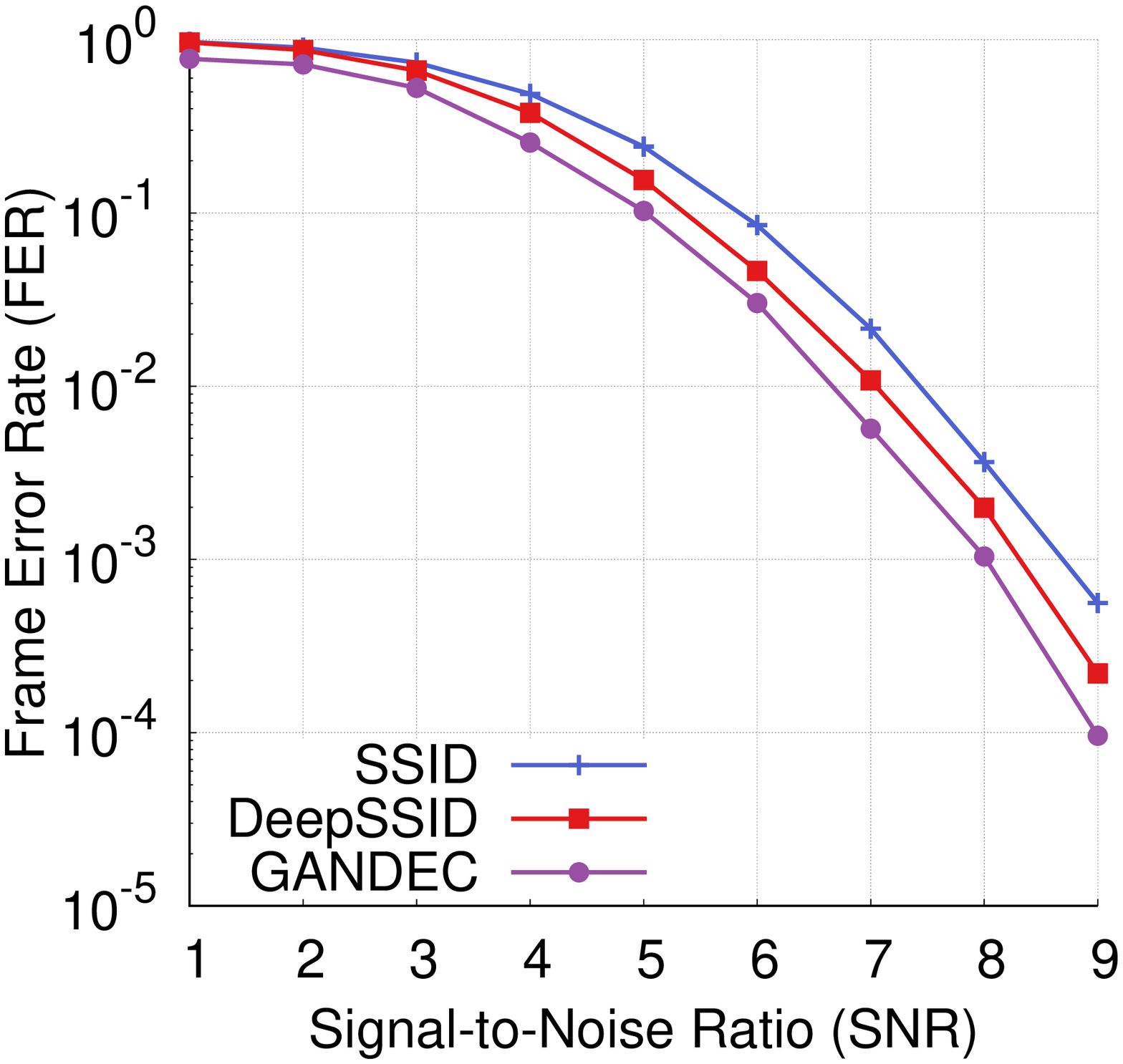}
		\caption{RS(15,11)}
	\end{subfigure}
	\begin{subfigure}[b]{0.42\linewidth}
		\includegraphics[width=\linewidth]{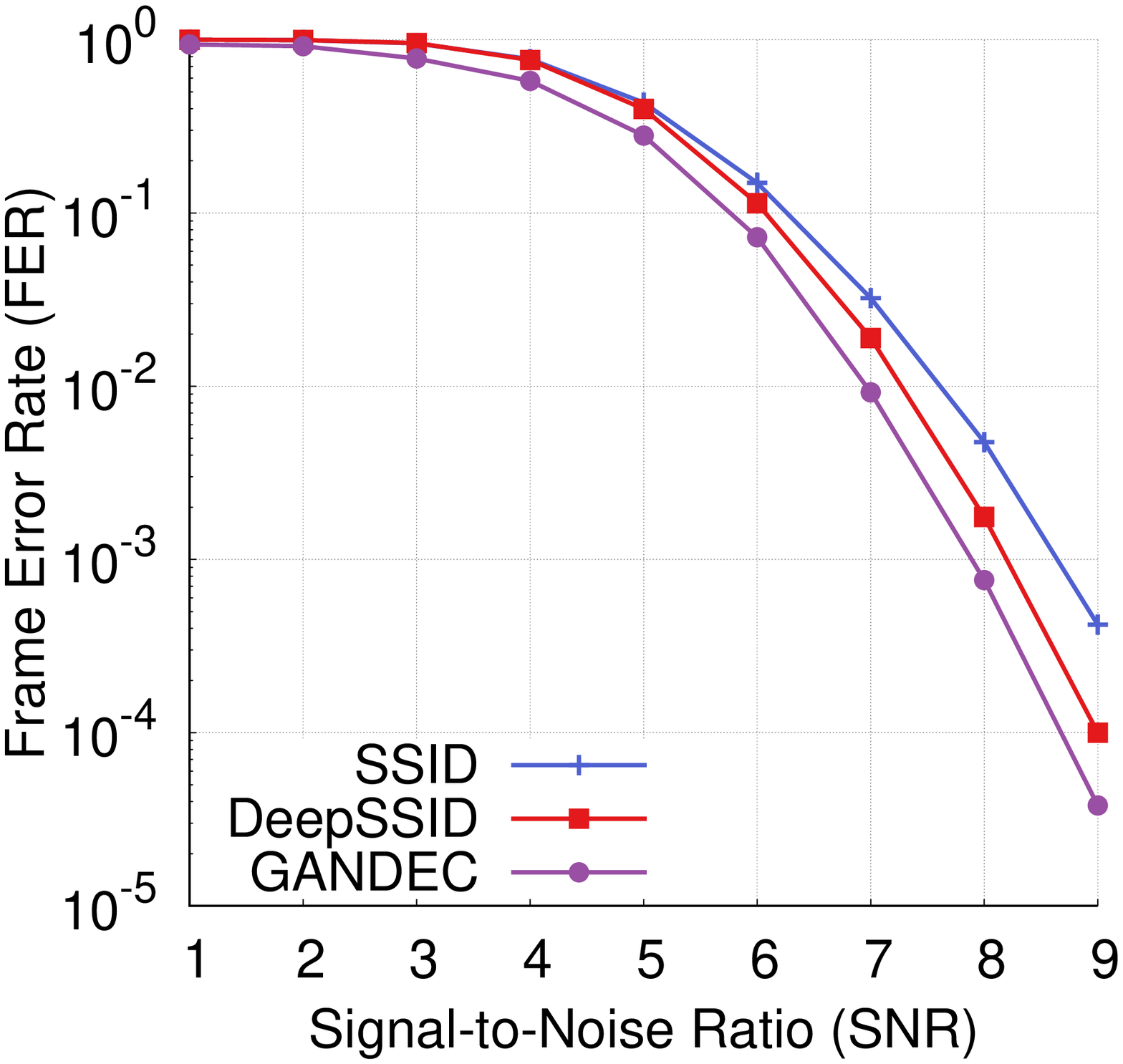}
		\caption{RS(31,27)}
	\end{subfigure}
	\caption{Performance of different decoders on Reed-Solomon (RS) codes. SSID is an effective belief propagation method for RS.}
	\label{fig:rs_fer}
\end{figure*}
\subsubsection{Unfolded BP Model Simplification}
\label{subsec:bp_fixed}

The number of weight parameters in the deep unfolding model is quite large, i.e., $O(L\cdot n + L\cdot E \cdot d_{max})$, where $L$ is the number of BP iterations, $E$ is the number of edges in the Tanner graph, and $d_{\max}$ denotes the maximum degree of any node.
We can improve the efficiency of our model by weight simplification.
We simplify the weight parameters $w_{i,v}$ and $w_{i,e,e'}$ of layer $i$ in the unfolding network into a single weight parameter $w_i$ and reduce the total number of trainable weights to $L$.
This simplification stems from our experimental observations that most of the parameters in the same layer have similar values after training.
Intuitively, this simplification helps reduce the computation in both training and testing, while resulting in similar performance as the full weight model.

\subsection{Online Training for Channel Dynamics Adaptation}

Compared to the decoder network by itself, our model has a strong advantageous characteristic that it only requires received noisy signals from the channel and uniformly random codewords, which can be generated using the generator matrix of the code.
This property makes our model amenable to online training and effectively adapting to channel dynamics.
Specifically, we can collect batches of received signals and generate random batches of codewords to train our model on the fly continually.
This contrasts with a typical decoder network, e.g., the deep unfolding model, which requires both the received signals and their corresponding transmitted codewords to train it, restricting their applicability in scenarios where the transmitted signals are not available at the receiver side. That is the common case in wireless communication with changing environment. Moreover, even if channel changes can be detected quickly and transmitted data can be obtained, retraining the decoder network could take a long time. Differently, our framework involves no such delay of obtaining data and retraining the decoder network as we can train online.




\section{Numerical Results}
\label{sec:exps}

We demonstrate the better decoding performance of our framework compared with the original decoder network and traditional decoding algorithms on various codes. Again, we consider the unfolded BP networks as the ML model in our experiments.

\subsection{Settings}
\textit{Error correcting codes and unfolded models:}
We run our experiments on two linear codes: BCH with BCH(15,11), BCH(63,45), and Reed-Solomon codes RS(15,11), RS(31,27), with an AWGN channel and various signal-to-noise ratios (SNRs).
For BCH codes, we use the neural offset min-sum deep unfolding model, FNOMS, as in \cite{nachmani2018deep} that unfolds 5 BP iterations and adds a discriminator with two fully connected layers, each of size 1024 with sigmoid activation function, and a single neuron output layer.
For Reed-Solomon codes, the decoder network, namely DeepSSID, unfolds 5 iterations of the Stochastic Shifting Based Iterative Decoding (SSID) model\cite{zhang2020iterative}. The same discriminator as above is also used for Reed-Solomon codes. Note that we deliberately use a very simple discriminator network here to  demonstrate the feasibility of deploying our framework in small devices with limited resources.

\textit{Train and test:}
We applied the Adam optimizer with a learning rate of 0.001 for training both decoder and discriminator and batch size of 64. 10000 and 100000 codewords are randomly drawn for training and testing, respectively. For our framework, we alternatively train decoder and discriminator as in a typical GAN training procedure.

\textit{Evaluation Metric:} We measure the performance of each method using Frame Error Rate (FER) as we focus on recovering the entire codeword or transmitted message.


\subsection{Results}

We compare our framework, denoted by \method, with the neural unfolded BP model and the BP decoder with 100 iterations.
The results are demonstrated in Figure~\ref{fig:bch_fer} for BCH codes and Figure~\ref{fig:rs_fer} for Reed-Solomon codes.
In both cases, the results show that \method{} achieves improved decoding performance compared to both previous deep unfolding networks and BP.
The improvement is approximately 0.5 dB compared to the unfolding model by itself, and 1 dB in comparison to BP decoder, consistently across all frame error rates and different codes. On the contrary, unfolded BP networks only improve over BP decoder for large SNR values.

Notably, the improvement of \method{} for shorter block length as in BCH(15,11) and RS(15,11) is more significant than for longer block length as in BCH(63,45) and RS(31,27).
This is possibly due to the fixed number of layers in the decoder network and suggests larger block codes may need more layers or deeper networks to perform well. Nevertheless, both \method{} and the unfolded BP model only correspond to 5 BP iterations while the BP decoder needs 100 iterations and the performance are still much lower.

\section{Conclusion}

We have proposed a game-theoretic framework that combines a decoder network with a discriminator, which compete in a zero-sum game.
The decoder tries to decode a received noisy signal from the communication channel while the discriminator tries to distinguish between codewords from the decoded words.
Both decoder and discriminator improve together to be better at their own tasks.
We show a strong connection between our model and the optimal maximum likelihood decoder as this decoder forms a Nash's equilibrium in our game model.
Numerical results confirm improved decoding performance compared to recent ML models and traditional decoders.

\bibliographystyle{IEEEtran}
\bibliography{relatedwork}

\end{document}